\documentclass[runningheads]{llncs}
\usepackage[misc]{ifsym} 
\usepackage{graphicx}
\usepackage[export]{adjustbox} 
\usepackage{subcaption} 
\usepackage{pgfplots}
\pgfplotsset{compat=newest}
\usepackage{paralist} 
\usepackage{amsthm}
\usepackage{amsmath}
\usepackage{amssymb}
\usepackage{relsize}
\usepackage{tikz} 
\usepackage{varwidth}
\usetikzlibrary{arrows.meta,arrows}
\usetikzlibrary{shapes, positioning}

\theoremstyle{plain}
\newtheorem{thm}{Theorem} 
\theoremstyle{definition}
\newtheorem{defn}{Definition} 
\newtheorem{exmp}{Example}
\newtheorem{corol}{Corollary}
\usepackage{url}
\usepackage{multirow}
\urldef{\mailsa}\path|author1@someinst.something|
\urldef{\mailsb}\path|author2@someotherinst.something|

\usepackage{slashbox}

\usepackage{multicol}
\usepackage{algorithm} 
\usepackage[ruled, vlined, linesnumbered,algo2e]{algorithm2e}
\usepackage[noend]{algpseudocode}
\usepackage{etoolbox}

\makeatletter
\newcommand*{\algrule}[1][\algorithmicindent]{%
  \makebox[#1][l]{%
    \hspace*{.2em}
    \vrule height .75\baselineskip depth .25\baselineskip
  }
}

\newcount\ALG@printindent@tempcnta
\def\ALG@printindent{%
    \ifnum \theALG@nested>0
    \ifx\ALG@text\ALG@x@notext
    \else
    \unskip
    \ALG@printindent@tempcnta=1
    \loop
    \algrule[\csname ALG@ind@\the\ALG@printindent@tempcnta\endcsname]%
    \advance \ALG@printindent@tempcnta 1
    \ifnum \ALG@printindent@tempcnta<\numexpr\theALG@nested+1\relax
    \repeat
    \fi
    \fi
}
\patchcmd{\ALG@doentity}{\noindent\hskip\ALG@tlm}{\ALG@printindent}{}{\errmessage{failed to patch}}
\patchcmd{\ALG@doentity}{\item[]\nointerlineskip}{}{}{} 
\makeatother

\newcommand{\RNum}[1]{\uppercase\expandafter{\romannumeral #1\relax}} 

\usepackage{listings}
\usepackage[toc,page]{appendix}
\usepackage[utf8]{inputenc}
\usepackage{graphicx}
\graphicspath{ {figures/} }
\usepackage{array}
\usepackage{hyperref}
\usepackage[nottoc,notlot,notlof]{tocbibind}

\usepackage{chngcntr}

\let\emptyset\varnothing

\usepackage{etoolbox}
\patchcmd{\paragraph}{\itshape}{\bfseries\boldmath}{}{}
  
\newcommand{\keywords}[1]{\par\addvspace\baselineskip
\noindent\keywordname\enspace\ignorespaces#1}

\begin{document}

\title{FIBS: A Generic Framework for Classifying Interval-based Temporal Sequences}
%
%
\author{S. Mohammad Mirbagheri\textsuperscript{(\Letter)} \and Howard J. Hamilton}
%
%
%
\institute{Department of Computer Science, University of Regina, Regina, Canada\\
\email {\{mirbaghs,Howard.Hamilton\}}@uregina.ca}

\maketitle              

\begin{abstract}
We study the problem of classifying interval-based temporal sequences (IBTSs). Since common classification algorithms cannot be directly applied to IBTSs, the main challenge is to define a set of features that effectively represents the data such that classifiers can be applied. Most prior work utilizes frequent pattern mining to define a feature set based on discovered patterns. However, frequent pattern mining is computationally expensive and often discovers many irrelevant patterns. To address this shortcoming, we propose the FIBS framework for classifying IBTSs. FIBS extracts features relevant to classification from IBTSs based on relative frequency and temporal relations. To avoid selecting irrelevant features, a filter-based selection strategy is incorporated into FIBS.
Our empirical evaluation on eight real-world datasets demonstrates the effectiveness of our methods in practice. The results provide evidence that FIBS effectively represents IBTSs for classification algorithms, which contributes to similar or significantly better accuracy compared to state-of-the-art competitors.
It also suggests that the feature selection strategy is beneficial to FIBS's performance.
\end{abstract}
\let\thefootnote\relax\footnotetext{This research was supported by funding from ISM Canada and the Natural Sciences and Engineering Research Council of Canada.}
\section{Introduction}
Interval-based temporal sequence (IBTS) data are collected from application domains in which events persist over intervals of time of varying lengths. Such domains include medicine \cite{2019medical,patelHepatit,medical}, sensor networks \cite{morchenSensor}, sign languages \cite{signLanguage}, and motion capture \cite{motion2016}. Applications that need to deal with this type of data are common in industrial, commercial, government, and health sectors. For example, some companies offer multiple service packages to customers that persist over varying periods of time and may be held concurrently. The sequence of packages that a customer holds can be represented as an IBTS.  

IBTSs can be obtained either directly from the applications or indirectly by data transformation. 
In particular, temporal abstraction of a univariate or multivariate time series may yield such data. Segmentation or aggregation of a time series into a succinct symbolic representation is called \textit{temporal abstraction} (TA) \cite{medical}. TA transforms a numerical time series to a symbolic time series. This high-level qualitative form of data provides a description of the raw time series data that is suitable for a human decision-maker (beacause it helps them to understand the data better) or for data mining. TA may be based on knowledge-based abstraction performed by a domain expert. An alternative is data-driven abstraction utilizing temporal discretization. Common unsupervised discretization methods are Equal Width, Symbolic Aggregate Approximation (SAX) \cite{SAX}, and Persist \cite{Persist}.  
Depending on the application scenario, symbolic time series may be categorized as point-based or as interval-based. Point-based data reflect scenarios in which events happen instantaneously or events are considered to have equal time intervals. Duration has no impact on extracting patterns for this type. Interval-based data, which is the focus of this study, reflect scenarios where events have unequal time intervals; here, duration plays an important role. 
Fig. \ref{obtain} depicts the process of obtaining interval-based temporal sequences.
\begin{figure}[!ht]
\centering
\includegraphics[scale=.09]{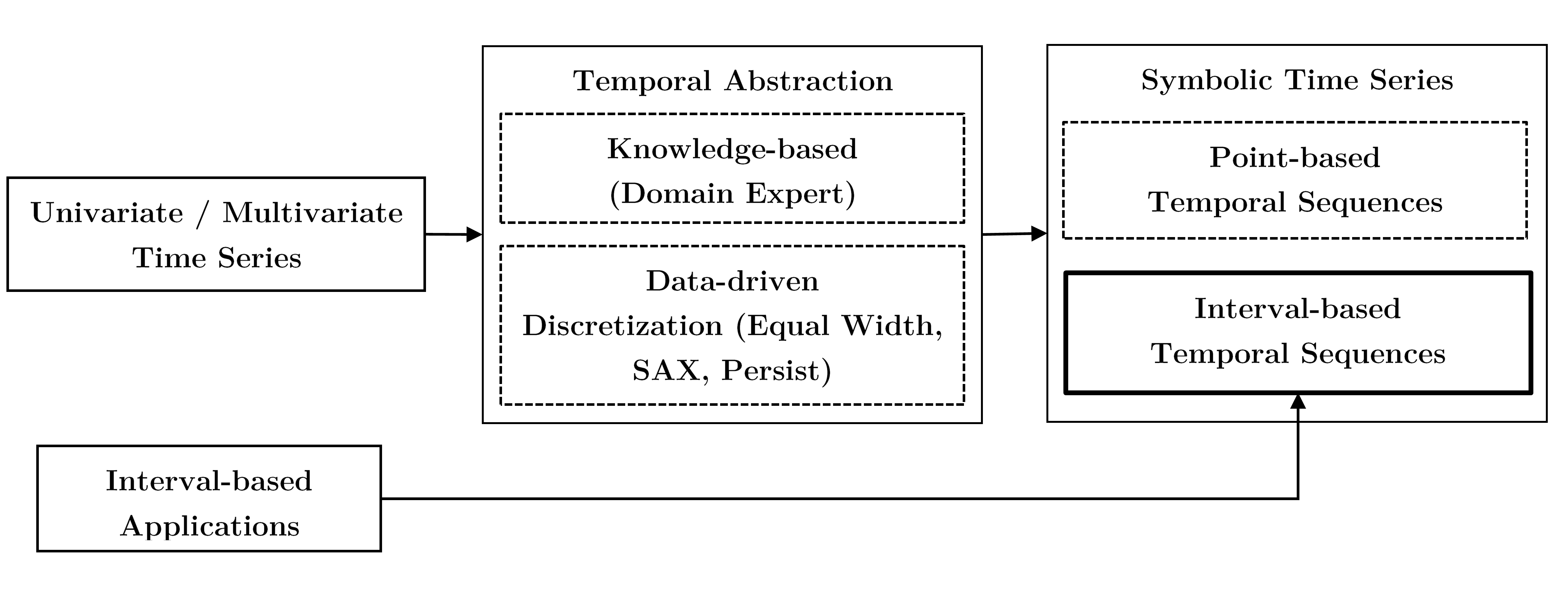}
\caption{ Process of obtaining interval-based temporal sequences}
\label{obtain}
\end{figure}

Classifying IBTSs is a relatively new research area. Although classification is an important machine learning task and has achieved great success in a wide range of applications (fields), the classification of IBTSs has not received much attention. A dataset of IBTSs contains longitudinal data where instances are described by a series of event intervals over time rather than features with a single value. Such dataset does not match the format required by standard classification algorithms to build predictive models.
Standard classification methods for multivariate time series (e.g., Hidden Markov Models \cite{Markov} and recurrent neural networks), time series similarity measures (e.g., Euclidean distance and Dynamic Time Warping (DTW) \cite{Warping}), and time series feature extraction methods (e.g., discrete Fourier transform, discrete wavelet transform and singular value decomposition) cannot be directly applied to such temporal data.

In this paper, we formalize the problem of classification of IBTSs based on feature-based classifiers and propose a new framework to solve this problem. The major contributions of this work are as follows: 
\begin{itemize}
\item We propose a generic framework named FIBS for classifying IBTSs. It represents IBTSs by extracting features relevant to classification from IBTSs based on relative frequency and temporal relations. 
\item To avoid selecting irrelevant features, we propose a heuristic filter-based feature selection strategy. FIBS utilizes this strategy to reduce the feature space and improve the classification accuracy.

\item We report on an experimental evaluation that shows the proposed framework is able to represent IBTSs effectively and classifying them efficiently.
\end{itemize}

The rest of the paper is organized as follows. Related work is presented in Section 2. Section 3 provides preliminaries and the problem statement. Section 4 presents the details of the FIBS framework and the feature selection strategy. Experimental results on real datasets and evaluation are given in Section 5. Section 6 presents conclusions.

\section{Related Work}
To date, only a few approaches to IBTS classification can be found in the literature. Most of them are domain-specific and based on frequent pattern mining techniques. 
Patel et al. \cite{patelHepatit} proposed the first work in the area of IBTS classification. They developed a method that first mines all frequent temporal patterns in an unsupervised setting and then uses some of these patterns as features for classification. They used Information Gain, a measure of discriminative power, as the selection criterion. After features are extracted, common classification techniques, such as decision trees or Support Vector Machines (SVM), are used to predict the classifications for unseen IBTSs.

Extracting features from frequent temporal patterns presents some challenges. Firstly, a well-known limitation of frequent pattern mining algorithms is that they extract too many frequent patterns, many of which are redundant or uninformative.
Several attempts have been made to address this limitation by discovering frequent temporal patterns in a supervised setting  \cite{batal2013,classification2015}. For example, Batal et al. \cite{batal2013} proposed the Minimal Temporal Patterns (MPTP) framework to filter out non-predictive and spurious temporal patterns to define a set of features for classifying Electronic Health Record (EHR) data. 
Secondly, discovering frequent patterns is computationally expensive. Lastly, classification based on features extracted from frequent patterns does not guarantee better performance than other methods.

In contrast to the approaches based on frequent pattern mining, a few studies offer similarity-based approaches to IBTS classification.
A robust similarity (or distance) measure allows machine learning tasks such as similarity search, clustering and classification to be performed. Towards this direction, Kostakis et al. \cite{Artemis} proposed two methods for comparing IBTSs. 
The first method, called the \textit{DTW-based} method, maps each IBTS to a set of vectors where one vector is created for each start- or end-point of any event-interval. The distances between the vectors are then computed using DTW.
The second method, called \textit{Artemis}, measures the similarity between two IBTSs based on the temporal relations that are shared between events. To do so, the IBTSs are mapped into a bipartite graph and the Hungarian algorithm is employed.
Kotsifakos et al. \cite{ibsm} proposed the Interval-Based Sequence Matching (\textit{IBSM}) method where each time point is represented by a binary vector that indicates which events are active at that particular time point. The distances between the vectors are then computed by Euclidean distance. In all three methods, IBTSs are classified by the $k$-NN classification algorithm and it was shown that IBSM outperforms the two other methods \cite{ibsm} with respect to both classification accuracy and runtime. 
Although the reported results are promising, such classifiers still suffer from major limitations. While Artemis ignores the duration of the event intervals, DWT-based and IBSM ignore the temporal relations between event intervals.

A feature-based framework for IBTS classification, called \textit{STIFE}, was recently proposed by Bornemann et al. \cite{2016stife}. STIFE extracts features using a combination of basic statistical metrics, shapelet \cite{shapelet} discovery and selection, and distance-based approaches. Then, a random forest is constructed using the extracted features to perform classification. It was shown that such a random forest achieves similar or better classification accuracy than $k$-NN using IBSM.

\section{Problem Statement}
We adapt definitions given in earlier research \cite{signLanguage} and describe the problem statement formally.

\begin{defn} (\textit{Event interval})
Let $\Sigma=\{\mathrm{A,B,...}\}$ denote a finite alphabet. A triple $e=(l,t_b,t_f)$ is called an event interval, where $l \in \Sigma$ is the event label and $t_b$, $t_f \in \mathbb{N}$, $(t_b<t_f)$ are the beginning and finishing times, respectively.
We also use $e.x$ to denote element $x$ of event interval $e$, e.g. $e.t_b$ is the beginning time of event interval $e$. The duration of event interval $e$ is $d(e)=t_f-t_b$.
\end{defn}
\begin{defn} (\textit{E-sequence})
An event-interval sequence or e-sequence $s=\langle e_1,e_2, \\ ...,e_m\rangle $ is a list of $m$ event intervals placed in ascending order based on their beginning times. If event intervals have equal beginning times, then they are ordered lexicographically by their labels. Multiple occurrences of an event are allowed in an e-sequence if they do not happen concurrently.
The duration of an e-sequence $s$ is $d(s)=max\{e_1.t_f, e_2.t_f, ..., e_m.t_f\}-min\{e_1.t_b,e_2.t_b, ..., e_m.t_b\}$. 
\end{defn}
\begin{defn} (\textit{E-sequence dataset})
An e-sequence dataset $D$ is a set of $n$ e-sequences $\{s_1,...,s_n\}$, where each e-sequence $s_i$ is associated with an unique identifier $i$.
\end{defn}
Table \ref{DB} depicts an e-sequence dataset consisting of four e-sequences with identifiers 1 to 4.

\begin{table}[H]
\centering
\caption{Example of an e-sequence dataset}
\label{DB}
 \setlength\extrarowheight{2pt} 
\begin{tabular}{|c|p{1cm}|p{1.7cm}|p{1.6cm}|l|}
\hline
\textbf{id} & \textbf{Event  Label} & \textbf{Beginning Time} & \textbf{Finishing  Time} &   \multicolumn{1}{>{\centering\arraybackslash}m{6cm}|}{\textbf{Event Sequence}}  \\ \hline
\multirow{4}{*}{1} & $A$ & 8& 28& \multirow{4}{*}{ \begin{tikzpicture} 
\draw (-2,0) -- node[above] {$A$} ++ (5.35,0) (1/3,-0.5)-- node[above] {$B$} ++(3/3,0) (2,-0.75)-- node[above] {$C$} ++(4/3,0) (2.3,-1.25)-- node[above] {$E$} ++(2/3,0);
 \end{tikzpicture}} \\ \cline{2-4}
        & $B$ & 18& 21& \\ \cline{2-4}
        & $C$ & 24& 28& \\ \cline{2-4}
        & $E$ & 25& 27& \\ \hline
\multirow{4}{*}{2} & $A$ & 1& 14& \multirow{4}{*}{ \begin{tikzpicture} 
\draw (-2,0) -- node[above] {$A$} ++ (3.35,0) (-2/3,-0.4)-- node[above] {$C$} ++(6/3,0) (-1/3,-0.85)-- node[above] {$E$} ++(3/3,0) (-1/3,-1.25)-- node[above] {$F$} ++(3/3,0);
 \end{tikzpicture} } \\ \cline{2-4}
        & $C$ & 6& 14& \\ \cline{2-4}
        & $E$ & 8& 11& \\ \cline{2-4}
        & $F$ & 8& 11& \\ \hline
\multirow{4}{*}{3} & $A$ & 6& 22& \multirow{4}{*}{ \begin{tikzpicture} 
\draw (-1.5,0) -- node[above] {$A$} ++ (4.65,0) (-1.5,-0.5)-- node[above] {$B$} ++(6/3,0) (.5,-0.75)-- node[above] {$C$} ++(6/3,0) (1.25,-1.25)-- node[above] {$E$} ++(2/3,0);
 \end{tikzpicture} } \\ \cline{2-4}
        & $B$ & 6& 14& \\ \cline{2-4}
        & $C$ & 14& 20& \\ \cline{2-4}
        & $E$ & 16& 18& \\ \hline
\multirow{5}{*}{4} & $A$ & 4& 24& \multirow{5}{*}{ \begin{tikzpicture} 
\draw (-2,0) -- node[above] {$A$} ++ (5.5,0) (-1.75,-0.4)-- node[above] {$B$} ++(4/3,0)
(-1.75,-0.9)-- node[above] {$D$} ++(6/3,0) (1.2,-1.2)-- node[above] {$C$} ++(6/3,0) (1.9,-1.7)-- node[above] {$E$} ++(2/3,0);
 \end{tikzpicture}} \\ \cline{2-4}
        & $B$ & 5& 10& \\ \cline{2-4}
        & $D$ & 5& 12& \\ \cline{2-4}
        & $C$ & 16& 22& \\ \cline{2-4}
        & $E$ & 18& 20& \\ \hline
	\end{tabular}
\end{table}

\paragraph{Problem Statement.}

Given an e-sequence dataset $D$, where each e-sequence is associated with a class label, the problem is to construct a representation of $D$ such that common feature-based classifiers are able to classify previously unseen e-sequences similar to those in $D$.  \\

\section{The FIBS Framework}
In this section, we introduce the FIBS framework for classifying e-sequence datasets. 
Many classification algorithms require data to be in a format reminiscent of a table, where rows represent instances (e-sequences) and columns represent features (attributes). Since an e-sequence dataset does not follow this format, we utilize FIBS to construct feature-based representations to enable standard classification algorithms to build predictive models. 

A \textit{feature-based representation} of a dataset has three components: a class label set, a feature set, and data instances. We first give a general definition of a feature-based representation based on these components \cite{feature}. 
\begin{defn}(\textit{Feature-based representation})
\label{ClassDFN}
A feature-based representation $K=(C,F,X)$ is defined as follows.
Let $C=\{c_1, c_2, ..., c_k\}$ be a set of $k$ class labels, $F=\{f_1, f_2, ..., f_z\}$ be a set of $z$ features (or attributes), $X=\{x_1,x_2, . . . ,x_n\}$ be a set of $n$ instances, and let $y_i \in C$ denote the class label of instance $x_i \in X$.
\end{defn}
In supervised settings, the set of class labels $C$ of the classes to which e-sequences belong is already known. Therefore, in order to form the feature-based representation, FIBS extracts the feature set $F$ and the instances $X$ from dataset $D$. To define the $F$ and $X$ components, we consider two alternative formulations based on relative frequency and temporal relations among the events. These formulations are explained in the following subsections.

\subsection{Relative Frequency}

\begin{defn} (\textit{Relative frequency})
The relative frequency $\mathrm{R}(s,l)$ of an event label $l \in \Sigma$ in an e-sequence $s \in D$, which is the duration-weighted frequency of the occurrences of $l$ in $s$, is defined as the accumulated durations of all event intervals with event label $l$ in $s$ divided by the duration of $s$. Formally:
\begin{equation}
\mathrm{R}(s,l)=\frac{1}{d(s)} \mathlarger{\sum}_{ e \in s  \ \wedge \ e.l=l  } d(e)
\end{equation} 
\end{defn}
Suppose that we want to specify a feature-based representation of an e-sequence dataset $D=\{s_1, s_2, ..., s_n\}$ using relative frequency. Let every unique event label $l \in \Sigma$ found in $D$ be used as a feature, i.e., let $F= \Sigma$.
Also let every e-sequence $s \in D$ be used as the basis for defining an instance $x \in X$. The feature-values of instance $x$ are specified as a vector containing the relative frequencies of every event label $l \in \Sigma$ in $s$. Formally, $X=\{x_1,x_2, . . . ,x_n\} \in \mathbb{R}^{n \times |\mathsmaller\Sigma|}$, $x_i=\langle \mathrm{R}(s_i,l_1), \mathrm{R}(s_i,l_2), ..., \mathrm{R}(s_i,l_{|\Sigma|}) \rangle.$ 
\begin{exmp}
Consider the feature-based representation that is constructed based on the relative frequency of the event labels in the e-sequence dataset shown in Table \ref{DB}.
Let the class label set $C$ be $\{+,-\}$ and the feature set $F$ be \{A, B, C, D, E, F\}. Assume that the class label of each of $s_1,s_3$, and $s_4$ is $+$ and the class label of $s_2$ is $-$. Table \ref{Relative} shows the resulting feature-based representation.
\label{ex1}
\end{exmp}

\begin{table}[]
\centering
\caption{Feature-based representation constructed based on relative frequency }
\label{Relative}
\setlength\tabcolsep{1.5pt}
\begin{tabular}{|c|c|c|c|c|c|c|}
\hline
\textbf{A}  &\textbf{B}  &\textbf{C}  &\textbf{D}  &\textbf{E}  &\textbf{F} & \textbf{Class} \\ \hline
1.00 & 0.15 & 0.20 & 0 & 0.10 & 0 & $+$	 \\ \hline	
1.00 & 0 & 0.62 & 0 & 0.23 & 0.23 & $-$	 \\ \hline	
1.00 & 0.50 & 0.38 & 0 & 0.13 & 0 & $+$	 \\ \hline	
1.00 & 0.25 & 0.30 & 0.35 & 0.10 & 0 & $+$	 \\ \hline	              
	\end{tabular}
\end{table}
\subsection{Temporal Relations}
Thirteen possible temporal relations between pairs of intervals were nicely categorized by Allen \cite{allen1983maintaining}. Table \ref{Allen} illustrates Allen's temporal relations. Ignoring the ``equals'' relation, six of the relations are inverses of the other six. We emphasize seven temporal relations, namely, equals (q), before (b), meets (m), overlaps (o), contains (c), starts (s), and finished-by (f), which we call the \textit{primary} temporal relations. Let set $U=T \cup I$ represents the thirteen temporal relation labels, where $T=\{q,b,m,o,c,s,f\}$ is the set of labels for the primary temporal relations and $I=\{t^{-1} \ | \ t \in T-\{q\} \}$ is the set of labels for the inverse temporal relations. \\
\begin{table}[]
\centering
\caption{Allen's temporal relations between two event intervals  }
\label{Allen}
\setlength\tabcolsep{1.5pt}
\begin{tabular}{|c|c|c|}
\hline
 \multicolumn{1}{|p{3.1cm}|}{\centering \textbf{ Primary \\ Temporal Relation}}  &\multicolumn{1}{p{3.1cm}|}{\centering \textbf{Inverse \\ Temporal Relation} }&  \multicolumn{1}{p{2.2cm}|}{\raggedleft  \textbf{Pictorial \\ Example} }\\ \hline
$\alpha$ equals $\beta$	& $\beta$ equals $\alpha$ &   \begin{tikzpicture} 
\draw (-1,0) -- node[above] {$\alpha$} ++ (1,0);
\draw(-1,-0.5)-- node[above] {$\beta$} ++(1,0);
\draw [dashed](-1,0.15) -- (-1,-0.65);
\draw [dashed](0,0.15) -- (0,-0.65);
 \end{tikzpicture}

         \\ \hline	
$\alpha$ before $\beta$	& $\beta$ after $\alpha$ &   \begin{tikzpicture} 
\draw (-1,0) -- node[above] {$\alpha$} ++ (1,0)(0.5,0)-- node[above] {$\beta$} ++(1.5,0);
 \end{tikzpicture}

         \\ \hline
$\alpha$ meets $\beta$	& $\beta$ met-by $\alpha$ &   \begin{tikzpicture} 
\draw (-1,0) -- node[above] {$\alpha$} ++ (1,0)(0,-0.5)-- node[above] {$\beta$} ++(1.5,0);
\draw [dashed](0,0.1) -- (0,-0.6);
 \end{tikzpicture}
         \\ \hline	
 $\alpha$ overlaps $\beta$	& $\beta$ overlapped-by $\alpha$ &   \begin{tikzpicture} 
\draw (-1,0) -- node[above] {$\alpha$} ++ (1,0);
\draw(-0.5,-0.5)-- node[above] {$\beta$} ++(1.5,0);
 \end{tikzpicture}
         \\ \hline	  
$\alpha$ contains $\beta$ & $\beta$ during $\alpha$	&   \begin{tikzpicture} 
\draw (-1,0) -- node[above] {$\alpha$} ++ (2,0);
\draw(-0.5,-0.5)-- node[above] {$\beta$} ++(1,0);
 \end{tikzpicture}
         \\ \hline	  
 $\alpha$ starts $\beta$ & $\beta$ startted-by $\alpha$	&   \begin{tikzpicture} 
\draw (-1,0) -- node[above] {$\alpha$} ++ (1,0);
\draw(-1,-0.5)-- node[above] {$\beta$} ++(2.5,0);
\draw [dashed](-1,0.1) -- (-1,-0.6);
 \end{tikzpicture}
         \\ \hline	  
         $\alpha$ finished-by $\beta$	& $\beta$ finishes $\alpha$ &   \begin{tikzpicture} 
\draw (1,-1) -- node[above] {$\beta$} ++ (1,0);
\draw(-0.5,-0.5)-- node[above] {$\alpha$} ++(2.5,0);
\draw [dashed](2,-0.4) -- (2,-1.1);
 \end{tikzpicture}
         \\ \hline	                 
	\end{tabular}
\end{table}

Exactly one of these relations holds between any ordered pair of event intervals. Some event labels may not occur in an e-sequence and some may occur multiple times. For simplicity, we assume the first occurrence of an event label in an e-sequence is more important than the remainder of its occurrences. Therefore, when extracting temporal relations from an e-sequence, only the first occurrence is considered and the rest are ignored. With this assumption, there are at most ${|\Sigma| \choose 2}$ possible pairs of event labels in a dataset.  

Based on Definition \ref{ClassDFN}, we now define a second feature-based representation, which relies on temporal relations.
Let $F={\Sigma \choose 2}$ be the set of all 2-combinations of event labels from $\Sigma$. The feature-values of instance $x_i$ are specified as a vector containing the labels corresponding to the temporal relations between every pair that occurs in an e-sequence $s_i$. In other words, $X=\{x_1,x_2, . . . ,x_n\} \in U^{n \times {|\Sigma| \choose 2}}$, where an instance $x_i \in X$ represents an e-sequence $s_i$.
\begin{exmp}
Following Example \ref{ex1}, Table \ref{temporal} shows the feature-based representation that is constructed based on the temporal relations between the pairs of event labels in the e-sequences given in Table \ref{DB}. To increase readability, 0 is used instead of $\emptyset$ to indicate that no temporal relation exists between the pair. 
\end{exmp}
\begin{table}[ht]
\centering
\caption{Feature-based representation constructed based on temporal relations }
\label{temporal}
\setlength\tabcolsep{1.5pt}
\begin{tabular}{|c|c|c|c|c|c|c|c|c|c|c|c|c|c|c|c|}
\hline
\textbf{A,B}  &\textbf{A,C}  &\textbf{A,D}  &\textbf{A,E}  &\textbf{A,F}  &\textbf{B,C} & \textbf{B,D} & \textbf{B,E}  &\textbf{B,F}  &\textbf{C,D}  &\textbf{C,E}  &\textbf{C,F}  &\textbf{D,E} &\textbf{D,F}&\textbf{E,F} &\textbf{Class} \\ \hline
c & f & 0 & c & 0 & b & 0&b &0 & 0 & c& 0 & 0&0&0 & $+$	 \\ \hline	
0 & f & 0 & c & c & 0 & 0&0 &0 & 0& c & c &0 &0& q& $-$	 \\ \hline	
$s^{-1}$ & c & 0 & c & 0 & b & 0 &b & 0 & 0 & c& 0& 0&0& 0&$+$	 \\ \hline	
c & c & c & c & 0 & b& s & b &0&$b^{-1}$&c&0&$b^{-1}$&0&0& $+$	 \\ \hline	              
	\end{tabular}
\end{table}

\subsection{Feature Selection} 

Feature selection for classification tasks aims to select a subset of features that are highly discriminative and thus contribute substantially to increasing the performance of the classification. Features with less discriminative power are undesirable since they either have little impact on the accuracy of the classification or may even harm it. As well, reducing the number of features improves the efficiency of many algorithms.

Based on their relevance to the targeted classes, features are divided by John et al. \cite{irrelevant} into three disjoint categories, namely, strongly relevant, weakly relevant, and irrelevant features. 
Suppose $f_i \in F$ and $\bar{f_i} = F - \{f_i\}$. Let $P(C \ | \ F)$ be the probability distribution of class labels in $C$ given the values for the features in $F$. The categories of feature relevance can be formalized as follows \cite {redundancy}. 
\begin{defn} (\textit{Strong relevance})
A feature $f_i$ is strongly relevant iff
\begin{equation}
P(C \ | \  f_i,\bar{f_i} \ ) \neq P(C \ | \ \bar{f_i})
\end{equation}
\end{defn}
\begin{defn} (\textit{Weak relevance})
A feature $f_i$ is weakly relevant iff
\begin{equation}
\begin{array}{l}
P(C \ | \ f_i,\bar{f_i}) = P(C \ | \ \bar{f_i}) \text{ and} \\
\exists \ g_i \subset \bar{f_i} \ \text{such that} \ P(C \ | \ f_i,g_i) \neq P(C \ | \ g_i)
\end{array}
\end{equation}
\end{defn}
\begin{corol} (\textit{Irrelevance}) \normalfont  A feature $f_i$ is irrelevant iff 
\begin{equation}
\forall \ g_i \subseteq \bar{f_i}, \ \ P(C \ | \ f_i,g_i) = P(C \ | \ g_i)
\end{equation}
\label{cor}
\end{corol} 
Strong relevance indicates that a feature is indispensable and it cannot be removed without loss of prediction accuracy. Weak relevance implies that the feature can sometimes contribute to prediction accuracy. Features are relevant if they are either strongly or weakly relevant and are irrelevant otherwise. Irrelevant features are dispensable and can never contribute to prediction accuracy. 

Feature selection is beneficial to the quality of the temporal relation representation, especially when there are many distinct event labels in the dataset. Although any feature selection method can be used to eliminate irrelevant features, some methods have advantages for particular representations.
Filter-based selection methods are generally efficient because they assess the relevance of features by examining intrinsic properties of the data prior to applying any classification method.  
We propose a simple and efficient filter-based method to avoid producing irrelevant features for the temporal relation representation. 
\subsection{Filter-based Feature Selection Strategy}
\label{subStrategy}
In this section, we propose a filter-based strategy for feature reduction that can also be used in unsupervised settings. We apply this strategy to avoid producing irrelevant features for the temporal relation representation.

\begin{thm}
\label{theorem1}
An event label $l$ is an irrelevant feature of an e-sequence dataset $D$ if its relative frequencies are equal in every e-sequence in dataset $D$. 
\end{thm}
\begin{proof}
Suppose event label $l$ occurs with equal relative frequencies in every e-sequence in dataset $D$. We construct a feature-based representation $K=(C,F,X)$ based on the relative frequencies of the event labels as previously described. Therefore, there exists a feature $f_i \in F$ that has the constant value of $v$ for all instances $x \in X$. We have $P(C \ | \ f_i)= P(C)$. Therefore, $P(C \ | \ f_i,g_i) = P(C \ | \ g_i)$. According to Corollary \ref{cor}, we conclude $f_i$ is an irrelevant feature.   
\end{proof}

We provide a definition for \textit{support} that is applicable to relative frequency.
 If we add up the relative frequencies of event label $l$ in all e-sequences of dataset $D$ and then normalize the sum, we obtain the support of $l$ in $D$. Formally:
\begin{equation}
sup(D,l)=\frac{1}{n} \sum_{s \in D}R(s,l)
\end{equation} 
where $n$ is the number of e-sequences in $D$. 

The support of an event label can be used as the basis of dimensionality reduction during pre-processing for a classification task. One can identify and discard irrelevant features (event labels) based on their supports. We will now show how the support is used to avoid extracting irrelevant features by the following corollary, which is an immediate consequence of Theorem \ref{theorem1}.
\begin{corol}
\label{corl2}
An event label $l$ whose support in dataset $D$ is 0 or 1 is an irrelevant feature. 
\end{corol}

\begin{proof}
As with the proof of Theorem \ref{theorem1}, assume we construct a feature-based representation $K$ based on the relative frequency of the event labels. If $sup(l,D)=0$ then, there exists a mapping feature $f_i \in F$ that has equal relative frequencies (values) of $0$ for all instances $x \in X$. The same argument holds if $sup(l,D)=1$. According to Theorem \ref{theorem1}, we conclude $f_i$ is an irrelevant feature. 
\end{proof} 

In practice, situations where the support of a feature is exactly 0 or 1 do not often happen. Hence, we propose a heuristic strategy that discards probably irrelevant features based on a confidence interval defined with respect to an error threshold $\epsilon$.       

\paragraph{Heuristic Strategy:}
If $sup(D,l)$ is not in a confidence interval $[0+\epsilon  , 1-\epsilon]$, then event label $l$ is presumably an irrelevant feature in $D$ and can be discarded.

\subsection{Comparison to Representation Based on Frequent Patterns}
In frequent pattern mining, the support of temporal pattern $p$ in a dataset is the number of instances that contain $p$. A pattern is frequent if its support is no less than a predefined threshold set by user. 
Once frequent patterns are discovered, after computationally expensive operations, a subset of frequent patterns are selected as features. The representation contains binary values such that if a selected pattern occurs in an e-sequence
the value of the corresponding feature is 1, and 0 otherwise. 
Example \ref{ex3} illustrates a limitation of classification of IBTSs based on frequent pattern mining where frequent patterns are irrelevant to the class labels.

\begin{exmp}
\label{ex3}
Consider Table \ref{DB} and its feature-based representation constructed based on relative frequency, as shown in Example \ref{ex1}. In this example, the most frequent pattern is A, which has a support of 1. However, according to Corollary \ref{corl2}, A is an irrelevant feature and can be discarded for the purpose of classification.
For this example, a better approach is to classify the e-sequences based on the presence or absence of F such that the occurrence of F in an e-sequence means the e-sequence belongs to the $-$ class and the absence of F means it belongs to the $+$ class.
\end{exmp}

In practice, the large number of frequent patterns affects the performance of the approach in both the pattern discovery step and the feature selection step. Obviously, mining patterns that are later found to be irrelevant, is useless and computationally costly. 

\section{Experiments}
In our experiments, we evaluate the effectiveness of the FIBS framework on the task of classifying interval-based temporal sequences using the well-known random forest classification algorithm on eight real world datasets. 
We evaluate performance of FIBS using classifiers implemented in R version 3.6.1. The FIBS framework was also implemented in R. All experiments were conducted on a laptop computer with a 2.2GHz Intel Core i5 CPU and 8GB memory. We obtain overall classification accuracy using 10-fold cross-validation. We also compare the results for FIBS against those for two
well-known methods, STIFE \cite{2016stife} and IBSM \cite{ibsm}.
In order to see the effect of the feature selection strategy, the FIBS framework was tested with it disabled (FIBS baseline) and with its error threshold $\epsilon$ set to various values.
 
\subsection{Datasets}
\label{DatasetSec} 
Eight real-world datasets from various application domains were used to evaluate the FIBS framework. Statistics concerning these datasets are summarized in Table \ref{StatDB}. More details about the datasets are as follows:
\begin{list}{\textbullet}{}
\item \textbf{ASL-BU} \cite{signLanguage}.
Event intervals correspond to facial or gestural expressions (e.g., head tilt right, rapid head shake, eyebrow raise, etc.) obtained from videos of American Sign Language expressions provided by Boston University. An e-sequence expresses an utterance using sign language that belongs to one of nine classes, such as wh-word, wh-question, verb, or noun. 
\item \textbf{ASL-BU2} \cite{signLanguage}. ASL-BU2 is a newer version of the ASL-BU dataset with improvements in annotation such that new e-sequences and additional event labels have been introduced. As above, an e-sequence expresses an utterance. 
\item \textbf{Auslan2} \cite{morchenSensor}. The e-sequences in the Australian Sign Language dataset contain event intervals that represent words like girl or right. 
\item \textbf{Blocks} \cite{morchenSensor}. Each event interval corresponds to a visual primitive obtained from videos of a human hand stacking colored blocks and describes which blocks are touched as well as the actions of the hand (e.g., contacts blue, attached hand red, etc.). Each e-sequence represents one of eight scenarios, such as assembling a tower.
\item \textbf{Context} \cite{morchenSensor}. Each event interval was derived from categorical and numeric data describing the context of a mobile device carried by a person in some situation (e.g., walking inside/outside, using elevator, etc). Each e-sequence represents one of five scenarios, such as being on a street or at a meeting. 
\item \textbf{Hepatitis} \cite{patelHepatit}. Each event interval represents the result of medical tests (e.g, normal, below or above the normal range, etc) during an interval. Each e-sequence corresponds to a series of tests over a period of 10 years that a patient who has either Hepatitis B or Hepatitis C undergoes.
\item \textbf{Pioneer} \cite{morchenSensor}. Event intervals were derived from the Pioneer-1 dataset available in the UCI repository corresponding to the input provided by the robot sensors (e.g, wheels velocity, distance to object, sonar depth reading, gripper state, etc). Each e-sequence in the dataset describes one of three scenarios: move, turn, or grip.
\item \textbf{Skating} \cite{morchenSensor}. Each event interval describes muscle activity and leg positions of one of six professional In-Line Speed Skaters during controlled tests at seven different speeds on a treadmill. Each e-sequence represents a complete movement cycle, which identifies one of the skaters. 
\end{list}

\begin{table}[H]
\centering
\caption{ Statistical information about datasets }
\label{StatDB}
\begin{tabular}{|c|c|c|c|c|c|c|c| }
\hline
 Dataset & \# e-sequences & \# Event Intervals & \multicolumn{3}{c|}{e-sequence Size} & $|\Sigma|$ & \# Classes   \\ 
  && & 	 min	& max & avg&&      \\ \hline
  ASL-BU  & 873  & 18,250& 4& 41& 18  &216& 9     \\ \hline 
  ASL-BU2 & 1839   & 2,447 &4 &93 &23 &254&  7     \\ \hline 
  Auslan2 & 200   & 2,447 &9 &20 & 12 &12 & 10     \\ \hline 
  Blocks & 210  & 1,207 &3 &12& 6 &8 & 8     \\ \hline
  Context & 240  & 19,355 & 47 &149&  81 & 54& 5      \\ \hline
  Hepatitis & 498  & 53,692 &15 &592& 108& 63 & 2      \\ \hline
  Pioneer & 160 & 8,949 & 36 & 89& 56& 92  & 3      \\ \hline
  Skating & 530  & 23,202 & 27& 143& 44& 41& 6      \\ \hline
\end{tabular}
\end{table}

\subsection{Performance Evaluation}
\label{Performance}
We assess the classification accuracy of FIBS on the set of datasets given in Section \ref{DatasetSec}, which is exactly the same set of datasets considered in work on IBSM and STIFE \cite{ibsm,2016stife}.
For a fair comparison and following \cite{2016stife}, in each case, we apply the random forest algorithm using FIBS to perform classifications.
We adopt the classification results of the IBSM and STIFE methods, as reported in Table 5 in \cite{2016stife}.  

Table \ref{accuracy} shows the mean classification accuracy on the datasets when using FIBS baseline, FIBS with the error threshold $\epsilon$ ranging from 0.01 to 0.03 (using the feature selection strategy defined in Section \ref{subStrategy}), 
STIFE, and IBSM. The best performance in each row is highlighted in bold. 
\begin{table}[htbp]
  \centering
  \caption{Mean classification accuracy of each framework on eight datasets. The last two rows indicate the mean and median results of each method across all datasets.}
    \begin{tabular}{|c|c|c|c|c|c|c|}
    \hline 
    Dataset &FIBS\_Baseline & FIBS\_0.01 & FIBS\_0.02 & FIBS\_0.03 &STIFE & IBSM \\ \hline 
    ASL-BU & \textbf{94.98} & 89.95 & 90.68 & 88.46 & 91.75 & 89.29 \\ \hline 
    ASL-BU2 & \textbf{94.43} & 92.39 & 92.67 & 93.61 & 87.49 & 76.92 \\ \hline 
    Auslan2 & 40.50  & 41.00    & 41.00    & 41.00    & \textbf{47.00} & 37.50 \\ \hline 
    Blocks & 100   & 100   & 100   & 100   & 100   & 100 \\ \hline 
    Context & 97.83 & 98.76 & 98.34 & 98.34 & \textbf{99.58} & 96.25 \\ \hline 
    HEPATITIS & 84.54 & \textbf{85.14} & 83.55 & 83.94 & 82.13 & 77.52 \\ \hline 
    Pioneer & \textbf{100} & \textbf{100} & \textbf{100} & \textbf{100} & 98.12 & 95.00 \\ \hline 
    Skating & 96.73 & 97.93 & 98.31 & \textbf{98.5} & 96.98 & 96.79 \\ \hline 
	\hline    
    Mean  & \textbf{88.63} & 88.15 & 88.07 & 87.98 & 87.88 & 83.66 \\ \hline 
    Median & 95.86 & 95.16 & 95.49 & \textbf{95.98} & 94.37 & 92.15 \\ \hline 
    \end{tabular}%
  \label{accuracy}%
\end{table}%

According to the Wilcoxon signed ranks tests applied across the results from the datasets given in Table \ref{accuracy}, each of the FIBS models has significantly higher accuracy than IBSM at significance level 0.05 (not shown). 
Compared with the STIFE framework, each of the FIBS models outperforms on four datasets, loses on three datasets, and ties on Blocks dataset at 100\%. The Wilcoxon signed ranks tests do not, however, confirm which method is significantly better. 
Overall, the results suggest that FIBS is a strong competitor in terms of accuracy. 
\subsection{Effect of Feature Selection}
The same experiments to classify the datasets using the random forest algorithm, which were given in Section \ref{Performance},
were conducted to determine the computational cost of FIBS with and without the feature selection strategy. 
Fig. \ref{Plots} shows the number of features produced by the frameworks and the execution time of applying the frameworks recorded on a log scale with base 10. The error threshold $\epsilon$ was varied from 0.00 (baseline) to 0.03 by 0.01.
\begin{figure}[H]
\centering
\begin{tabular}[c]{cc}
\begin{subfigure}[c]{\linewidth}
\includegraphics[width=\textwidth , height=0.25\textheight]{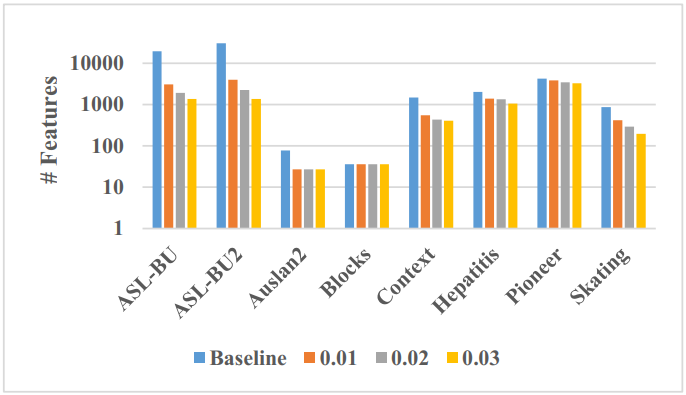}
\subcaption{ Number of features}
\label{Nfeature}
\end{subfigure}
\\
\begin{subfigure}[c]{\linewidth}
\includegraphics[width=\textwidth , height=0.25\textheight]{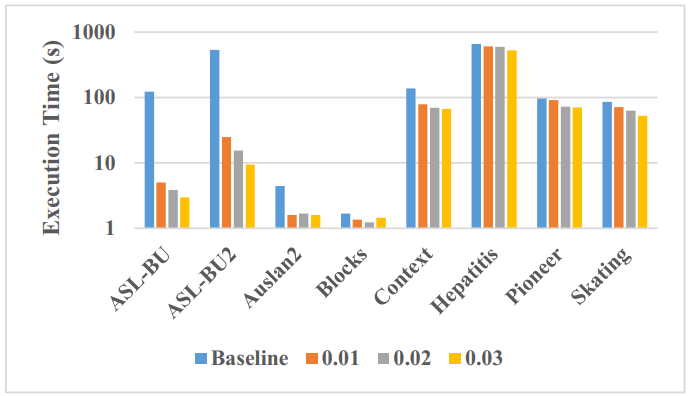}
\subcaption{Execution time (s)}
\label{TFeature}
\end{subfigure}
\end{tabular}  
\caption{Effect of feature selection strategy on the eight datasets based on different error thresholds $\epsilon$:  (\subref{Nfeature}) number of generated features on a log scale with base 10, (\subref{TFeature}) execution time (s) on a log scale with base 10.}
\label{Plots}
\end{figure}

As shown in Fig. \ref{Nfeature}, applying the feature selection strategy reduces the number of features, and consequently decreases the execution time in all datasets (Fig. \ref{TFeature}).
In particular, due to a significant reduction in the number of irrelevant features for ASL-BU and ASL-BU2, applying the FIBS framework with the strategy achieves over an order of magnitude speedup compared to FIBS without the strategy.
As shown by the mean classification accuracy of the models in Table \ref{accuracy}, applying the strategy also either improves the accuracy of the classification or does not have a significant adverse effect on it in all datasets. This result was confirmed by the Wilcoxon signed ranks tests at significance level 0.05 (not shown). 
Overall, the above results suggest that incorporating the feature selection strategy into FIBS is beneficial.
\section{Conclusion}
To date, most attempts to classify interval-based temporal sequences (IBTSs) have been performed in frameworks based on frequent pattern mining. As a simpler alternative, we propose a feature-based framework, called FIBS, for classifying IBTSs. FIBS incorporates two possible representations for features extracted from IBTSs, one based on the relative frequency of the occurrences of event labels and the other based on the temporal relations among the event intervals. Due to the possibility of generating too many features when using the latter representation, we proposed a heuristic feature selection strategy based on the idea of the support for the event labels. The experimental results demonstrated that methods implemented in the FIBS framework can achieve significantly better or similar performance in terms of accuracy when classifying IBTSs compared to the state-of-the-art competitors.
These results provide evidence that the FIBS framework effectively represents IBTS data for classification algorithms.

When extracting temporal relations among multiple occurrences of events with the same label in an e-sequence, FIBS considers only the first occurrences. In the future, the impact of temporal relations among such events could be studied under various assumptions. 

 \bibliographystyle{splncs} 
 
{\let\chapter\section\bibliography{bibfile} 





\end{document}
